\documentclass{article}
\usepackage{ijcai19}

\usepackage{soul}
\usepackage{url}
\usepackage[hidelinks]{hyperref}
\usepackage[utf8]{inputenc}
\usepackage[small]{caption}
\usepackage{booktabs}
\usepackage{algorithm}
\usepackage{algorithmicx}
\usepackage[noend]{algpseudocode}
\usepackage{amsmath}
\usepackage{amssymb}
\usepackage{amsthm}
\usepackage{bbm}
\usepackage{bm}
\usepackage{color}
\usepackage{dirtytalk}
\usepackage{dsfont}
\usepackage{enumerate}
\usepackage{graphicx}
\usepackage{mathtools}
\usepackage{subfigure}
\usepackage{times}

\usepackage[usenames,dvipsnames]{xcolor}
\usepackage[capitalize,noabbrev]{cleveref}

\newcommand{\commentout}[1]{}
\newcommand{\junk}[1]{}
\newcommand{\citet}[1]{\citeauthor{#1}~\shortcite{#1}}

\newtheorem{theorem}{Theorem}

\newtheorem{lemma}{Lemma}

\newcommand{\cA}{\mathcal{A}}

\newcommand{\cH}{\mathcal{H}}
\newcommand{\eps}{\varepsilon}

\newcommand{\realset}{\mathbb{R}}

\newcommand{\E}[1]{\mathbb{E} \left[#1\right]}
\newcommand{\condE}[2]{\mathbb{E} \left[#1 \,\middle|\, #2\right]}

\newcommand{\prob}[1]{\mathbb{P} \left(#1\right)}
\newcommand{\condprob}[2]{\mathbb{P} \left(#1 \,\middle|\, #2\right)}

\newcommand{\var}[1]{\mathrm{var} \left[#1\right]}

\newcommand{\abs}[1]{\left|#1\right|}
\newcommand{\ceils}[1]{\left\lceil#1\right\rceil}
\newcommand*\dif{\mathop{}\!\mathrm{d}}
\newcommand{\floors}[1]{\left\lfloor#1\right\rfloor}
\newcommand{\I}[1]{\mathds{1} \! \left\{#1\right\}}

\newcommand{\set}[1]{\left\{#1\right\}}

\DeclareMathOperator*{\argmax}{arg\,max\,}
\DeclareMathOperator*{\argmin}{arg\,min\,}

\mathchardef\mhyphen="2D

\newcommand{\fpl}{{\tt FPL}}
\newcommand{\giro}{{\tt Giro}}

\newcommand{\klucb}{{\tt KL\mhyphen UCB}}

\newcommand{\phe}{{\tt PHE}}

\newcommand{\ts}{{\tt TS}}
\newcommand{\ucb}{{\tt UCB1}}

\title{Perturbed-History Exploration in Stochastic Multi-Armed Bandits}

\author{
Branislav Kveton$^1$ \and
Csaba Szepesv\'ari$^{2, 3}$ \and
Mohammad Ghavamzadeh$^4$ \And
Craig Boutilier$^1$ \\
\affiliations
$^1$Google Research \\
$^2$DeepMind \\
$^3$ University of Alberta \\
$^4$Facebook AI Research
\emails
\{bkveton, cboutilier\}@google.com,
szepesva@cs.ualberta.ca,
mgh@fb.com
}

\begin{document}

\maketitle

\begin{abstract}
We propose an online algorithm for cumulative regret minimization in a stochastic multi-armed bandit. The algorithm adds $O(t)$ i.i.d.\ \emph{pseudo-rewards} to its history in round $t$ and then pulls the arm with the highest average reward in its perturbed history. Therefore, we call it \emph{perturbed-history exploration} ($\phe$). The pseudo-rewards are carefully designed to offset potentially underestimated mean rewards of arms with a high probability. We derive near-optimal gap-dependent and gap-free bounds on the $n$-round regret of $\phe$. The key step in our analysis is a novel argument that shows that randomized Bernoulli rewards lead to optimism. Finally, we empirically evaluate $\phe$ and show that it is competitive with state-of-the-art baselines.
\end{abstract}


\section{Introduction}
\label{sec:introduction}

A \emph{multi-armed bandit} \cite{lai85asymptotically,auer02finitetime,lattimore19bandit} is an online learning problem where actions of the \emph{learning agent} are represented by \emph{arms}. After the arm is \emph{pulled}, the agent receives its \emph{stochastic reward}. The objective of the agent is to maximize its expected cumulative reward. The agent does not know the mean rewards of the arms in advance and faces the so-called \emph{exploration-exploitation dilemma}: \emph{explore}, and learn more about the arm; or \emph{exploit}, and pull the arm with the highest average reward thus far. The \emph{arm} may be a treatment in a clinical trial and its \emph{reward} is the outcome of that treatment on some patient population.

\emph{Thompson sampling (TS)} \cite{thompson33likelihood,russo18tutorial} and \emph{optimism in the face of uncertainty (OFU)} \cite{auer02finitetime,dani08stochastic,abbasi-yadkori11improved} are the most celebrated and studied exploration strategies in stochastic multi-armed bandits. These strategies are near optimal in multi-armed \cite{garivier11klucb,agrawal13further} and linear \cite{abbasi-yadkori11improved,agrawal13thompson} bandits. However, they typically do not generalize easily to complex problems. For instance, in generalized linear bandits \cite{filippi10parametric}, we only know how to construct \emph{approximate} high-probability confidence sets and posterior distributions. These approximations affect the statistical efficiency of bandit algorithms \cite{filippi10parametric,zhang16online,abeille17linear,jun17scalable,li17provably}. In online learning to rank \cite{radlinski08learning}, we only have statistically efficient algorithms for simple user interaction models, such as the cascade model \cite{kveton15cascading,katariya16dcm}. If the model was a general graphical model with latent variables \cite{chapelle09dynamic}, we would not know how to design a bandit algorithm with regret guarantees. In general, efficient approximations to high-probability confidence sets and posterior distributions are hard to design \cite{gopalan14thompson,kawale15efficient,lu17ensemble,riquelme18deep,lipton18bbq,liu18customized}.

In this work, we propose a novel exploration strategy that is conceptually straightforward and has the potential to easily generalize to complex problems. In round $t$, the learning agent adds $O(t)$ \emph{i.i.d.\ pseudo-rewards} to its history and treats them as if they were generated by actual arm pulls. Then the agent pulls the arm with the highest average reward in this \emph{perturbed history} and observes the reward of the pulled arm. The pseudo-rewards are drawn from the same family of distributions as actual rewards, but generate \emph{maximum variance} randomized data.

Our algorithm, \emph{perturbed-history exploration} ($\phe$), is inherently \emph{optimistic}. To see this, note that the lack of \say{optimism} regarding arm $i$ in round $t$, that its estimated mean reward is below the actual mean, is due to a specific history of past $O(t)$ rewards. These rewards are independent noisy realizations of the mean reward of arm $i$. Therefore, the lack of optimism can be offset by adding $O(t)$ i.i.d.\ pseudo-rewards to the history of arm $i$, so that the estimated mean reward of arm $i$ in its perturbed history is above the mean with a high probability. This design is conceptually simple and appealing, because maximum variance rewards can be easily generated for any reward generalization model.

We make the following contributions in this paper. First, we propose $\phe$, a multi-armed bandit algorithm where the mean rewards of arms are estimated using a mixture of actual rewards and i.i.d.\ pseudo-rewards. Second, we analyze $\phe$ in a $K$-armed bandit with $[0, 1]$ rewards, and prove both $O(K \Delta^{-1} \log n)$ and $O(\sqrt{K n \log n})$ bounds on its $n$-round regret, where $\Delta$ is the minimum gap between the mean rewards of the optimal and suboptimal arms. The key to our analysis is a novel argument that shows that randomized Bernoulli rewards lead to optimism. Finally, we empirically compare $\phe$ to several baselines and show that it is competitive with the best of them.


\section{Setting}
\label{sec:setting}

We use the following notation. The set $\set{1, \dots, n}$ is denoted by $[n]$. We define $\mathrm{Ber}(x; p) = p^x (1 - p)^{1 - x}$ and let $\mathrm{Ber}(p)$ be the corresponding Bernoulli distribution. We also define $B(x; n, p) = \binom{n}{x} p^x (1 - p)^{n - x}$ and let $B(n, p)$ be the corresponding binomial distribution. For any event $E$, $\I{E} = 1$ if and only if event $E$ occurs, and is zero otherwise.

We study the problem of cumulative regret minimization in a stochastic multi-armed bandit. Formally, a \emph{stochastic multi-armed bandit} \cite{lai85asymptotically,auer02finitetime,lattimore19bandit} is an online learning problem where the learning agent sequentially pulls $K$ arms in $n$ rounds. In round $t \in [n]$, the agent pulls arm $I_t \in [K]$ and receives its reward. The reward of arm $i \in [K]$ in round $t$, $Y_{i, t}$, is drawn i.i.d.\ from a distribution of arm $i$, $P_i$, with mean $\mu_i$ and support $[0, 1]$. The goal of the agent is to maximize its expected cumulative reward in $n$ rounds. The agent does not know the mean rewards of the arms in advance and learns them by pulling the arms.

Without loss of generality, we assume that the first arm is \emph{optimal}, that is $\mu_1 > \max_{i > 1} \mu_i$. Let $\Delta_i = \mu_1 - \mu_i$ denote the \emph{gap} of arm $i$. Maximization of the expected cumulative reward in $n$ rounds is equivalent to minimizing the \emph{expected $n$-round regret}, which we define as
\begin{align*}
  R(n)
  = \sum_{i = 2}^K \Delta_i \E{\sum_{t = 1}^n \I{I_t = i}}\,.
\end{align*}


\section{Perturbed-History Exploration}
\label{sec:phe}

Our new algorithm, \emph{perturbed-history exploration} ($\phe$), is presented in \cref{alg:phe}. $\phe$ pulls the arm with the highest average reward in its perturbed history, which is estimated as follows. Let $T_{i, t} = \sum_{\ell = 1}^t \I{I_\ell = i}$ denote the number of pulls of arm $i$ in the first $t$ rounds and $s = T_{i, t - 1}$. Then the estimated reward of arm $i$ in round $t$, $\hat{\mu}_{i, t}$, is the average of its past $s$ \emph{rewards} and $a s$ i.i.d.\ \emph{pseudo-rewards} $(Z_\ell)_{\ell = 1}^{a s}$, for some tunable integer $a > 0$. In line \ref{alg:phe:value}, $\hat{\mu}_{i, t}$ is computed from the sum of the rewards of arm $i$ after $s$ pulls, $V_{i, s}$, and the sum of its pseudo-rewards, $U_{i, s}$. After the arm is pulled, the cumulative reward of that arm is updated with its reward in round $t$ (line \ref{alg:phe:update}). All arms are initially pulled once (line \ref{alg:phe:initialization}).

$\phe$ can be implemented computationally efficiently, such that its computational cost in round $t$ does not depend on $t$. The key observation is that the sum of $a s$ Bernoulli random variables with mean $1 / 2$ is a sample from a binomial distribution with mean $a s / 2$. Therefore, $U_{i, s} \sim B(a s, 1 / 2)$.

The \emph{perturbation scale $a$} is the only tunable parameter of $\phe$ (line \ref{alg:phe:perturbation scale}), which dictates the number of pseudo-rewards that are added to the perturbed history. Therefore, $a$ controls the trade-off between exploration and exploitation. In particular, higher values of $a$ lead to more exploration. We argue informally below that any $a > 1$ suffices for sublinear regret. We prove in \cref{sec:analysis} that any $a > 2$ guarantees it.

\begin{algorithm}[t]
  \caption{Perturbed-history exploration in a multi-armed bandit with $[0, 1]$ rewards}
  \label{alg:phe}
  \begin{algorithmic}[1]
    \State \textbf{Inputs}: Perturbation scale $a$
    \label{alg:phe:perturbation scale}
    \Statex
    \For{$i = 1, \dots, K$}
    \Comment{Initialization}
      \State $T_{i, 0} \gets 0, \, V_{i, 0} \gets 0$
    \EndFor
    \For{$t = 1, \dots, n$}
      \For{$i = 1, \dots, K$}
      \Comment{Estimate mean arm rewards}
        \If{$T_{i, t - 1} > 0$}
          \State $s \gets T_{i, t - 1}$
          \State $\displaystyle U_{i, s} \gets \sum_{\ell = 1}^{a s} Z_\ell$, where
          $(Z_\ell)_{\ell = 1}^{a s} \sim \mathrm{Ber}(1 / 2)$
          \State $\displaystyle \hat{\mu}_{i, t} \gets
          \frac{V_{i, s} + U_{i, s}}{(a + 1) s}$
          \label{alg:phe:value}
        \Else
          \State $\hat{\mu}_{i, t} \gets + \infty$
          \label{alg:phe:initialization}
        \EndIf
      \EndFor
      \State $I_t \gets \argmax_{i \in [K]} \hat{\mu}_{i, t}$
      \Comment{Pulled arm}
      \State Pull arm $I_t$ and get reward $Y_{I_t, t}$
      \Statex
      \For{$i = 1, \dots, K$}
      \Comment{Update statistics}
        \If{$i = I_t$}
          \State $T_{i, t} \gets T_{i, t - 1} + 1$
          \State $V_{i, T_{i, t}} \gets V_{i, T_{i, t - 1}} + Y_{i, t}$
          \label{alg:phe:update}
        \Else
          \State $T_{i, t} \gets T_{i, t - 1}$
        \EndIf
      \EndFor
    \EndFor
  \end{algorithmic}
\end{algorithm}

Now we examine how exploration emerges within our algorithm. Fix arm $i$ and the number of its pulls $s$. Let $V_{i, s}$ be the cumulative reward of arm $i$ after $s$ pulls. Let $(Z_\ell)_{\ell = 1}^{a s} \sim \mathrm{Ber}(1 / 2)$ be $a s$ i.i.d.\ pseudo-rewards and $U_{i, s} = \sum_{\ell = 1}^{a s} Z_\ell$ denote their sum. Then the mean reward of arm $i$ (line \ref{alg:phe:value}) is estimated as
\begin{align}
  \hat{\mu}
  = \frac{V_{i, s} + U_{i, s}}{(a + 1) s}\,.
  \label{eq:phe estimator}
\end{align}
This estimator has two key properties that allow us to bound the regret of $\phe$ in \cref{sec:analysis}. First, it \emph{concentrates} at the scaled and shifted mean reward of arm $i$. More precisely, let $\bar{U}_{i, s} = \E{U_{i, s}}$ and $\bar{V}_{i, s} = \E{V_{i, s}}$. Then we have
\begin{align}
  \E{\hat{\mu}}
  & = \frac{\bar{V}_{i, s} + \bar{U}_{i, s}}{(a + 1) s}
  = \frac{\mu_i + a / 2}{a + 1}\,,
  \label{eq:perturbed mean} \\
  \var{\hat{\mu}}
  & \leq \frac{\sigma_{\max}^2}{(a + 1) s}\,,
  \label{eq:perturbed variance}
\end{align}
where $\sigma_{\max}^2$ is the maximum variance of any random variable on $[0, 1]$. By Popoviciu's inequality on variances \cite{popoviciu35variance}, we have $\sigma_{\max}^2 = 1 / 4$, which is precisely the variance of $Z \sim \mathrm{Ber}(1 / 2)$.

Second, $\hat{\mu}$ is sufficiently \emph{optimistic} in the following sense. Let $E = \set{\bar{V}_{i, s} / s - V_{i, s} / s = \eps}$ be the event that the estimated mean reward of arm $i$ is below the mean by $\eps > 0$. We say that $\hat{\mu}$ is \emph{optimistic} if
\begin{align}
  \condprob{\frac{V_{i, s} + U_{i, s}}{(a + 1) s} \geq
  \frac{\bar{V}_{i, s} + \bar{U}_{i, s}}{(a + 1) s}}{E}
  > \prob{E}
  \label{eq:phe informal optimism}
\end{align}
for any $\eps > 0$ such that $\prob{E} > 0$. That is, for any deviation $\eps > 0$, the conditional probability that the randomized mean reward $\hat{\mu}$ is at least as high as $\E{\hat{\mu}}$ is higher than the probability of that deviation. Under this condition, $\phe$ explores enough and can escape potentially harmful deviations.

Now we argue informally that \eqref{eq:phe informal optimism} holds for $a > 1$ in $\phe$. Fix any $\eps > 0$. First, note that
\begin{align*}
  \prob{E}
  = \prob{\frac{\bar{V}_{i, s}}{s} - \frac{V_{i, s}}{s} = \eps}
  \leq \prob{\frac{\bar{V}_{i, s}}{s} - \frac{V_{i, s}}{s} \geq \eps}
\end{align*}
and
\begin{align*}
  & \condprob{\frac{V_{i, s} + U_{i, s}}{(a + 1) s} \geq
  \frac{\bar{V}_{i, s} + \bar{U}_{i, s}}{(a + 1) s}}{E} \\
  & \quad = \condprob{\frac{V_{i, s} + U_{i, s}}{(a + 1) s} -
  \frac{V_{i, s} + \bar{U}_{i, s}}{(a + 1) s} \geq \frac{\eps}{a + 1}}{E} \\
  & \quad = \prob{\frac{U_{i, s}}{s} - \frac{\bar{U}_{i, s}}{s} \geq \eps}\,.
\end{align*}
The last equality holds because $U_{i, s} - \bar{U}_{i, s}$ is independent of past rewards. Based on the above two inequalities, \eqref{eq:phe informal optimism} holds when
\begin{align}
  \prob{\frac{U_{i, s}}{s} - \frac{\bar{U}_{i, s}}{s} \geq \eps}
  > \prob{\frac{\bar{V}_{i, s}}{s} - \frac{V_{i, s}}{s} \geq \eps}\,.
  \label{eq:phe tail optimism}
\end{align}
Finally, if both $V_{i, s} / s$ and $U_{i, s} / s$ were normally distributed, \eqref{eq:phe tail optimism} would hold if the variance of $V_{i, s} / s$ was lower than that of $U_{i, s} / s$. This is indeed true, since
\begin{align*}
  \var{V_{i, s} / s}
  \leq \sigma_{\max}^2 / s\,, \quad
  \var{U_{i, s} / s}
  = a \sigma_{\max}^2 / s\,;
\end{align*}
and $a > 1$ from our assumption. This concludes our informal argument. We evaluate $\phe$ with $a > 1$ in \cref{sec:experiments}.


\section{Analysis}
\label{sec:analysis}

$\phe$ is an instance of general randomized exploration in Section 3 of \citet{kveton19garbage}. So, the regret of $\phe$ can be bounded using their Theorem 1, which we restate below.

\begin{theorem}
\label{thm:gre regret bound} For any $(\tau_i)_{i = 2}^K \in \realset^{K - 1}$, the expected $n$-round regret of Algorithm 1 in \citet{kveton19garbage} can be bounded from above as $R(n) \leq \sum_{i = 2}^K \Delta_i (a_i + b_i)$, where
\begin{align*}
  a_i
  & = \sum_{s = 0}^{n - 1} \E{\min \set{1 / Q_{1, s}(\tau_i) - 1, n}}\,, \\
  b_i
  & = \sum_{s = 0}^{n - 1} \prob{Q_{i, s}(\tau_i) > 1 / n} + 1\,.
\end{align*}
\end{theorem}

\noindent For any arm $i$ and the number of its pulls $s \in [n] \cup \set{0}$,
\begin{align*}
  Q_{i, s}(\tau)
  = \condprob{\hat{\mu} \geq \tau}{\hat{\mu} \sim p(\cH_{i, s}), \, \cH_{i, s}}
\end{align*}
is the tail probability that the estimated mean reward of arm $i$, $\hat{\mu}$, is at least $\tau$ conditioned on the history of the arm after $s$ pulls, $\cH_{i, s}$; where $p$ is the sampling distribution of $\hat{\mu}$ and $\tau$ is a tunable parameter. In $\phe$, the history $\cH_{i, s}$ is $V_{i, s}$ and $\hat{\mu}$ is defined in \eqref{eq:phe estimator}. Following \citet{kveton19garbage}, we set $\tau_i$ in \cref{thm:gre regret bound} to the average of the scaled and shifted mean rewards of arms $1$ and $i$,
\begin{align*}
  \tau_i
  = \frac{\mu_i + a / 2}{a + 1} + \frac{\Delta_i}{2 (a + 1)}\,,
\end{align*}
which are defined in \eqref{eq:perturbed mean}. This setting leads to the following gap-dependent regret bound.

\begin{theorem}
\label{thm:phe regret bound} For any $a > 2$, the expected $n$-round regret of $\phe$ is bounded as
\begin{align*}
  R(n)
  \leq \sum_{i = 2}^K \Delta_i
  \bigg(\underbrace{\frac{16 a c}{\Delta_i^2} \log n + 2}_
  {\text{\emph{$a_i$ in \cref{thm:gre regret bound}}}} +
  \underbrace{\frac{8 a}{\Delta_i^2} \log n + 3}_
  {\text{\emph{$b_i$ in \cref{thm:gre regret bound}}}}\bigg)\,,
\end{align*}
where
\begin{align}
  c
  = \frac{e^2 \sqrt{2 a}}{\sqrt{\pi}} \exp\left[\frac{16}{a - 2}\right]
  \left(1 + \sqrt{\frac{\pi a}{8 (a - 2)}}\right)\,.
  \label{eq:c}
\end{align}
\end{theorem}
\begin{proof}
The proof has two parts. In \cref{sec:b_i upper bound}, we prove an upper bound on $b_i$ in \cref{thm:gre regret bound}. In \cref{sec:a_i upper bound}, we prove an upper bound on $a_i$ in \cref{thm:gre regret bound}. Finally, we add these upper bounds for all arms $i > 0$.
\end{proof}

\noindent A standard reduction yields a gap-free regret bound.

\begin{theorem}
\label{thm:phe gap-free regret bound} For any $a > 2$, the expected $n$-round regret of $\phe$ is bounded as
\begin{align*}
  R(n)
  \leq 4 \sqrt{2 a (2 c + 1) K n \log n} + 5 K\,,
\end{align*}
where $c$ is defined in \cref{thm:phe regret bound}.
\end{theorem}
\begin{proof}
Let $\cA = \set{i \in [K]: \Delta_i \geq \varepsilon}$ be the set of arms whose gaps are at least $\varepsilon > 0$. Then by the same argument as in the proof of \cref{thm:phe regret bound} and from the definition of $\cA$, we have
\begin{align*}
  R(n)
  & \leq \sum_{i \in \cA} \frac{8 a (2 c + 1)}{\Delta_i} \log n +
  \varepsilon n + 5 \abs{\cA} \\
  & \leq \frac{8 a (2 c + 1) K}{\varepsilon} \log n +
  \varepsilon n + 5 K\,.
\end{align*}
Now we choose $\displaystyle \varepsilon = \sqrt{\frac{8 a (2 c + 1) K \log n}{n}}$, which completes the proof.
\end{proof}

\subsection{Discussion}
\label{sec:discussion}

We derive two regret bounds. The gap-dependent bound in \cref{thm:phe regret bound} is $O(K \Delta^{-1} \log n)$, where $\Delta = \min_{i > 1} \Delta_i$ is the minimum gap, $K$ is the number of arms, and $n$ is the number of rounds. This scaling is considered near optimal in stochastic multi-armed bandits. The gap-free bound in \cref{thm:phe gap-free regret bound} is $O(\sqrt{K n \log n})$. This scaling is again near optimal, up to the factor of $\sqrt{\log n}$, in stochastic multi-armed bandits.

A potentially large factor in our bounds is $\exp[16 / (a - 2)]$ in \eqref{eq:c}. It arises in the lower bound on the probability of a binomial tail (\cref{sec:technical lemmas}) and is likely to be loose. Nevertheless, it is constant in $K$, $\Delta$, and $n$; and decreases significantly even for small $a$. For instance, it is only $e^4$ at $a = 6$.

\subsection{Upper Bound on $b_i$ in \cref{thm:gre regret bound}}
\label{sec:b_i upper bound}

Fix arm $i > 1$. Based on our choices of $\cH_{i, s}$, $\hat{\mu}$, and $\tau_i$, we have for $s > 0$ that
\begin{align*}
  Q_{i, s}(\tau_i)
  = \condprob{\frac{V_{i, s} + U_{i, s}}{(a + 1) s} \geq
  \frac{\mu_i + a / 2 + \Delta_i / 2}{a + 1}}{V_{i, s}}\,.
\end{align*}
We set $Q_{i, 0}(\tau_i) = 1$, because of the optimistic initialization in line \ref{alg:phe:initialization} of $\phe$. We abbreviate $Q_{i, s}(\tau_i)$ as $Q_{i, s}$.

Fix the number of pulls $s$ and let $m = 8 a \Delta_i^{-2} \log n$. If $s \leq m$, we bound $\prob{Q_{i, s} > 1 / n}$ trivially by $1$. If $s > m$, we split our proof based on the event that $V_{i, s}$ is not much larger than its expectation,
\begin{align*}
  E
  = \set{V_{i, s} - \mu_i s \leq \Delta_i s / 4}\,.
\end{align*}
On event $E$,
\begin{align*}
  Q_{i, s}
  & = \condprob{V_{i, s} + U_{i, s} - \mu_i s - \frac{a s}{2} \geq
  \frac{\Delta_i s}{2}}{V_{i, s}} \\
  & \leq \condprob{U_{i, s} - \frac{a s}{2} \geq \frac{\Delta_i s}{4}}{V_{i, s}} \\
  & \leq \exp\left[- \frac{\Delta_i^2 s}{8 a}\right]
  \leq n^{-1}\,,
\end{align*}
where the first inequality is by the definition of event $E$, the second is by Hoeffding's inequality, and the last is from $s > m$. On the other hand, event $\bar{E}$ is unlikely because
\begin{align*}
  \prob{\bar{E}}
  \leq \exp\left[- \frac{\Delta_i^2 s}{8}\right]
  \leq \exp\left[- \frac{\Delta_i^2 s}{8 a}\right]
  \leq n^{-1}\,,
\end{align*}
where the first inequality is from Hoeffding's inequality, the second is from $a > 1$, and the last is from $s > m$. Now we apply the last two inequalities and get
\begin{align*}
  \prob{Q_{i, s} > 1 / n}
  = {} & \E{\condprob{Q_{i, s} > 1 / n}{V_{i, s}} \I{E}} + {} \\
  & \E{\condprob{Q_{i, s} > 1 / n}{V_{i, s}} \I{\bar{E}}} \\
  \leq {} & 0 + \prob{\bar{E}}
  \leq n^{-1}\,.
\end{align*}
Finally, we chain our upper bounds for all $s$ and get
\begin{align*}
  b_i
  \leq 1 + \sum_{s = 0}^{\floors{m}} 1 +
  \smashoperator[r]{\sum_{s = \floors{m} + 1}^{n - 1}} n^{-1}
  \leq \frac{8 a}{\Delta_i^2} \log n + 3\,.
\end{align*}
This completes our proof.

\subsection{Upper Bound on $a_i$ in \cref{thm:gre regret bound}}
\label{sec:a_i upper bound}

Fix arm $i > 1$. Based on our choices of $\cH_{1, s}$, $\hat{\mu}$, and $\tau_i$, we have for $s > 0$ that
\begin{align*}
  Q_{1, s}(\tau_i)
  = \condprob{\frac{V_{1, s} + U_{1, s}}{(a + 1) s} \geq
  \frac{\mu_1 + a / 2 - \Delta_i / 2}{a + 1}}{V_{1, s}}\,.
\end{align*}
We set $Q_{1, 0}(\tau_i) = 1$, because of the optimistic initialization in line \ref{alg:phe:initialization} of $\phe$. We abbreviate $Q_{1, s}(\tau_i)$ as $Q_{1, s}$, and define $F_s = 1 / Q_{1, s} - 1$.

Fix the number of pulls $s$ and let $m = 16 a \Delta_i^{-2} \log n$. If $s = 0$, $Q_{1, s} = 1$ and we obtain $\E{\min \set{F_s, n}} = 0$. Now consider the case of $s > 0$. If $s \leq m$, we apply the upper bound in \cref{thm:optimism upper bound} in \cref{sec:technical lemmas} and get
\begin{align*}
  & \E{\min \set{F_s, n}}
  \leq \E{1 / Q_{1, s}} \\
  & \quad \leq \E{1 / \condprob{V_{1, s} + U_{1, s} \geq \mu_1 s + a s / 2}
  {V_{1, s}}}
  \leq c\,,
\end{align*}
where $c$ is defined in \eqref{eq:c}. Note that $a$ in \cref{thm:optimism upper bound} plays the role of $a / 2$ in this claim.

If $s > m$, we split our argument based on the event that $V_{1, s}$ is not much smaller than its expectation,
\begin{align*}
  E
  = \set{\mu_1 s - V_{1, s} \leq \Delta_i s / 4}\,.
\end{align*}
On event $E$,
\begin{align*}
  Q_{1, s}
  & = \condprob{\mu_1 s + \frac{a s}{2} - V_{1, s} - U_{1, s} \leq
  \frac{\Delta_i s}{2}}{V_{1, s}} \\
  & \geq \condprob{\frac{a s}{2} - U_{1, s} \leq \frac{\Delta_i s}{4}}{V_{1, s}} \\
  & = 1 - \condprob{\frac{a s}{2} - U_{1, s} > \frac{\Delta_i s}{4}}{V_{1, s}} \\
  & \geq 1 - \exp\left[- \frac{\Delta_i^2 s}{8 a}\right] 
  \geq \frac{n^2 - 1}{n^2}\,,
\end{align*}
where the first inequality is by the definition of event $E$, the second is by Hoeffding's inequality, and the last is from $s > m$. This lower bound yields
\begin{align*}
  F_s
  = \frac{1}{Q_{1, s}} - 1
  \leq \frac{n^2}{n^2 - 1} - 1
  = \frac{1}{n^2 - 1}
  \leq n^{-1}
\end{align*}
for $n \geq 2$. On the other hand, event $\bar{E}$ is unlikely because
\begin{align*}
  \prob{\bar{E}}
  \leq \exp\left[- \frac{\Delta_i^2 s}{8}\right]
  \leq \exp\left[- \frac{\Delta_i^2 s}{8 a}\right]
  \leq n^{-2}\,,
\end{align*}
where the first inequality is from Hoeffding's inequality, the second is from $a > 1$, and the last is from $s > m$. Now we apply the last two inequalities and get
\begin{align*}
  \E{\min \set{F_s, n}}
  = {} & \E{\condE{\min \set{F_s, n}}{V_{1, s}} \I{E}} + {} \\
  & \E{\condE{\min \set{F_s, n}}{V_{1, s}} \I{\bar{E}}} \\
  \leq {} & n^{-1} \prob{E} + n \, \prob{\bar{E}}
  \leq 2 n^{-1}\,.
\end{align*}
Finally, we chain our upper bounds for all $s$ and get
\begin{align*}
  a_i
  \leq 0 + \sum_{s = 1}^{\floors{m}} c +
  \smashoperator[r]{\sum_{s = \floors{m} + 1}^{n - 1}} 2 n^{-1}
  \leq \frac{16 a c}{\Delta_i^2} \log n + 2\,.
\end{align*}
This completes our proof.


\section{Experiments}
\label{sec:experiments}

\begin{figure*}[t]
  \centering
  \raisebox{-0.85in}{\includegraphics[width=4.675in]{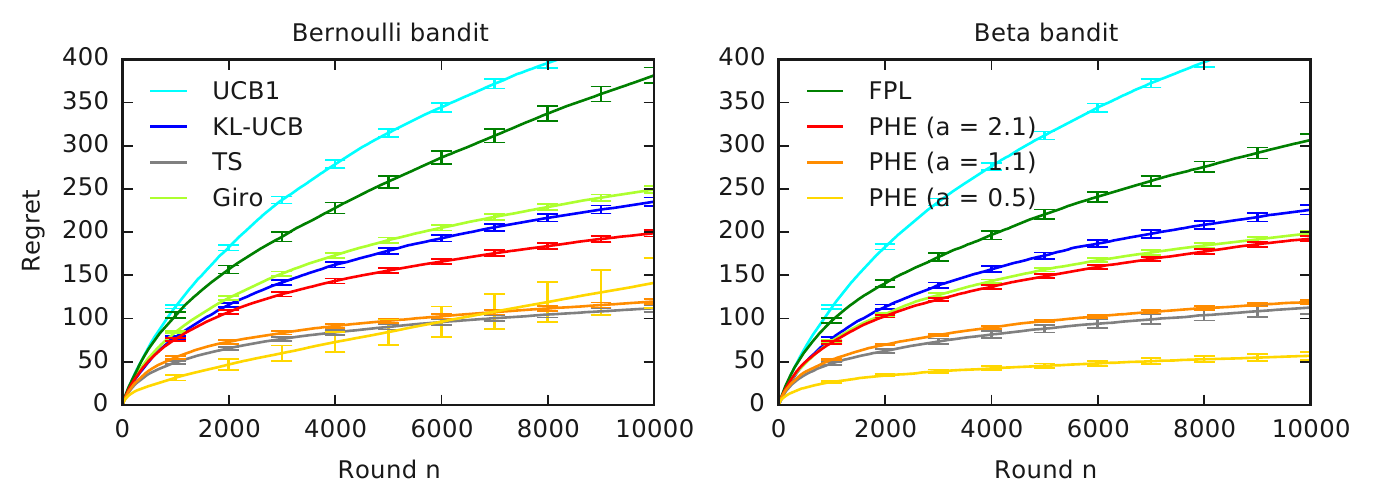}} \quad
  {\small
  \begin{tabular}{r r|r r r} \hline
    \multicolumn{2}{c|}{Model} & \multicolumn{3}{|c}{Run time (seconds)} \\
    $K$ & $n$ & $\ts$ & $\phe$ & $\giro$ \\ \hline
    $5$ & $1$k & $13.9$ & $18.2$ & $48.7$ \\
    $10$ & $1$k & $13.5$ & $17.0$ & $83.1$ \\
    $20$ & $1$k & $14.7$ & $19.1$ & $157.0$ \\
    $5$ & $10$k & $134.7$ & $179.5$ & $843.5$ \\
    $10$ & $10$k & $146.0$ & $180.3$ & $1250.2$ \\
    $20$ & $10$k & $136.3$ & $182.6$ & $1916.5$ \\ \hline
  \end{tabular}
  } \\
  \hspace{0.25in} (a) \hspace{2in} (b) \hspace{2.1in} (c)
  \caption{{\bf a}. Comparison of $\phe$ to multiple baselines in a Bernoulli bandit. {\bf b}. Comparison of $\phe$ to multiple baselines in a beta bandit. {\bf c}. Run times of three randomized algorithms in a beta bandit. All results are averaged over $100$ randomly chosen problems.}
  \label{fig:results}
\end{figure*}

We compare $\phe$ to five baselines: $\ucb$ \cite{auer02finitetime}, $\klucb$ \cite{garivier11klucb}, Bernoulli $\ts$ \cite{agrawal13further} with a $\mathrm{Beta}(1, 1)$ prior, $\giro$ \cite{kveton19garbage}, and $\fpl$ \cite{neu13efficient}. The baselines are chosen for the following reasons. $\klucb$ and $\ts$ are statistically near-optimal in Bernoulli bandits. We implement them with $[0, 1]$ rewards as follows. For any observed reward $Y_{i, t} \in [0, 1]$, we draw $\hat{Y}_{i, t} \sim \mathrm{Ber}(Y_{i, t})$ and then use it instead of $Y_{i, t}$ \cite{agrawal13further}. $\giro$ is chosen because it explores similarly to $\phe$, by adding pseudo-rewards to its history (\cref{sec:related work}). We implement it with $a = 1$, as analyzed in \citet{kveton19garbage}. $\fpl$ is chosen because it perturbs the estimates of mean rewards similarly to $\phe$ (\cref{sec:related work}). We implement it with geometric resampling and exponential noise, as described in \citet{neu13efficient}.

We experiment with three settings of perturbation scales $a$ in $\phe$: $2.1$, $1.1$, and $0.5$. The first value is greater than $2$ and is formally justified in \cref{sec:analysis}. The second value is greater than $1$ and is informally justified in \cref{sec:phe}. The last value is used to illustrate that the regret of $\phe$ can be linear when $\phe$ is not parameterized properly.

To run $\phe$ with a non-integer perturbation scale $a$, we replace $a s$ in $\phe$ with $\ceils{a s}$. The analysis of $\phe$ in \cref{sec:analysis} can be extended to this setting. We also experimented with $a = 1$ and $a = 2$. We do not report these results because they are similar to those at $a = 1.1$ and $a = 2.1$.

\subsection{Comparison to Baselines}
\label{sec:comparison to baselines}

In the first experiment, we evaluate $\phe$ on two classes of the bandit problems in \citet{kveton19garbage}. The first class is a Bernoulli bandit where $P_i = \mathrm{Ber}(\mu_i)$. The second class is a \emph{beta bandit} where $P_i = \mathrm{Beta}(v \mu_i, v (1 - \mu_i))$ and $v = 4$. We experiment with $100$ randomly chosen problems in each class. Each problem has $K = 10$ arms and the mean rewards of these arms are chosen uniformly at random from interval $[0.25, 0.75]$. The horizon is $n = 10000$ rounds.

Our results are reported in Figures \ref{fig:results}a and \ref{fig:results}b. We observe that $\phe$ with $a > 1$ outperforms four of our baselines: $\ucb$, $\klucb$, $\giro$, and $\fpl$. This is unexpected, since the design of $\phe$ is conceptually simple; and neither requires nor uses confidence intervals or posteriors. $\phe$ becomes competitive with $\ts$ at $a = 1.1$. Note that the regret of $\phe$ is linear in the Bernoulli bandit at $a = 0.5$. This shows that our suggestions for setting the perturbation scale $a$ are reasonably tight.

\subsection{Computational Cost}
\label{sec:computational cost}

In the second experiment, we compare the run times of three randomized algorithms: $\ts$, which samples from a beta posterior; $\giro$, which bootstraps from a history with pseudo-rewards; and $\phe$, which samples pseudo-rewards from a binomial distribution. The number of arms is between $5$ and $20$, and the horizon is up to $n = 10000$ rounds.

Our results are reported in \cref{fig:results}c. In all settings, the run time of $\phe$ is comparable to that of $\ts$. The run time of $\giro$ is an order of magnitude higher. The reason is that the computational cost of bootstrapping grows linearly with the number of past observations.


\section{Related Work}
\label{sec:related work}

Our algorithm design bears a similarity to three existing designs, which we discuss in detail below.

$\giro$ is a bandit algorithm where the mean reward of the arm is estimated by its average reward in a \emph{bootstrap sample} of its history with pseudo-rewards \cite{kveton19garbage}. The algorithm has a provably sublinear regret in a Bernoulli bandit. $\phe$ improves over $\giro$ in three respects. First, its design is simpler, because $\phe$ merely adds random pseudo-rewards and does not bootstrap. Second, $\phe$ has a sublinear regret in \emph{any} $K$-armed bandit with $[0, 1]$ rewards. Third, $\phe$ is computationally efficient beyond a Bernoulli bandit. We discuss this in \cref{sec:phe}.

Our work is also closely related to posterior sampling. In particular, let $\mu \sim \mathcal{N}(\mu_0, \sigma^2)$ and $(Y_\ell)_{\ell = 1}^s \sim \mathcal{N}(\mu, \sigma^2)$ be $s$ i.i.d.\ noisy observations of $\mu$. Then the posterior distribution of $\mu$ conditioned on $(Y_\ell)_{\ell = 1}^s$ is
\begin{align}
  \mathcal{N}\left(\frac{\mu_0 + \sum_{\ell = 1}^s Y_\ell}{s + 1}, \,
  \frac{\sigma^2}{s + 1}\right)\,.
  \label{eq:normal posterior}
\end{align}
A sample from this distribution can be also drawn as follows. First, draw $s + 1$ i.i.d.\ samples $(Z_\ell)_{\ell = 0}^s \sim \mathcal{N}(0, \sigma^2)$. Then
\begin{align*}
  \frac{\mu_0 + \sum_{\ell = 1}^s Y_\ell + \sum_{\ell = 0}^s Z_\ell}{s + 1}
\end{align*}
is a sample from \eqref{eq:normal posterior}. Unfortunately, the above equivalence holds only for normal random variables. Therefore, it cannot justify $\phe$ as a form of Thompson sampling. Nevertheless, the scale of the perturbation is similar to \eqref{eq:phe estimator}, which suggests that $\phe$ is sound.

\emph{Follow the perturbed leader (FPL)} \cite{hannan57approximation,kalai05efficient} is an algorithm design where the learning agent pulls the arm with the lowest perturbed cumulative cost. In our notation, $I_t = \argmin_{i \in [K]} \tilde{V}_{i, t - 1} + \tilde{U}_{i, t}$, where $\tilde{V}_{i, t - 1}$ is the cumulative cost of arm $i$ in the first $t - 1$ rounds and $\tilde{U}_{i, t}$ is the perturbation of arm $i$ in round $t$. $\phe$ differs from FPL in three respects. First, $\tilde{U}_{i, t} = O(\sqrt{n})$ in FPL. In $\phe$, the noise in round $t$ \emph{adapts} to the number of arm pulls, because $U_{i, T_{i, t - 1}} = O(T_{i, t - 1})$. Second, FPL has been traditionally studied in the \emph{non-stochastic full-information} setting. In comparison, $\phe$ is designed for the \emph{stochastic bandit} setting. \citet{neu13efficient} extended FPL to the bandit setting using geometric resampling and we compare to their algorithm in \cref{sec:experiments}. Finally, all existing FPL analyses derive \emph{gap-free regret bounds}. We derive a \emph{gap-dependent regret bound}.


\section{Conclusions}
\label{sec:conclusions}

We propose a new online algorithm, $\phe$, for cumulative regret minimization in stochastic multi-armed bandits. The key idea in $\phe$ is to add $O(t)$ i.i.d.\ pseudo-rewards to the history in round $t$ and then pull the arm with the highest average reward in this perturbed history. The pseudo-rewards are drawn from the maximum variance distribution. We derive $O(K \Delta^{-1} \log n)$ and $O(\sqrt{K n \log n})$ bounds on the $n$-round regret of $\phe$, where $K$ is the number of arms and $\Delta$ is the minimum gap between the mean rewards of the optimal and suboptimal arms. This result is unexpected, since the design of $\phe$ is conceptually simple. We empirically compare $\phe$ to several baselines and show that it is competitive with the best of them.

$\phe$ can be easily adapted to any reward distributions with a bounded support. If $Y_{i, t} \in [m, M]$, $Y_{i, t}$ in line \ref{alg:phe:update} of $\phe$ should be replaced with $(Y_{i, t} - m) / (M - m)$.

$\phe$ can be applied to structured problems, such as generalized linear bandits \cite{filippi10parametric}, as follows. Let $x_i$ be the feature vector of arm $i$. Then $((x_{I_\ell}, Y_{I_\ell, \ell}))_{\ell = 1}^{t - 1}$ is the \emph{history} in round $t$ and a natural choice for the \emph{pseudo-history} is $((x_{I_\ell}, Z_{j, \ell}))_{j \in [a], \, \ell \in [t - 1]}$, where $Z_{j, \ell} \sim \mathrm{Ber}(1 / 2)$ are i.i.d.\ random variables. In round $t$, the learning agent fits a reward generalization model to a mixture of both histories and pulls the arm with the highest estimated reward in that model. We leave the analysis and empirical evaluation of this algorithm for future work. The algorithm was analyzed in a linear bandit in \citet{kveton19perturbed2}.

We believe that $\phe$ can be extended to other perturbation schemes. For instance, since $\var{V_{i, s}} \leq s / 4$, it is plausible that any $s$ i.i.d.\ pseudo-rewards with a comparable variance, such as $(Z_\ell)_{\ell = 1}^s \sim \mathcal{N}(0, 1 / 4)$, would lead to optimism. We leave the analyses of such designs for future work.

\appendix


\section{Technical Lemmas}
\label{sec:technical lemmas}

Fix arm $i$ and the number of its pulls $n$. Let $X$ be the cumulative reward of arm $i$ after $n$ pulls and $Y = \sum_{\ell = 1}^{2 a n} Z_\ell$ be the sum of $2 a n$ i.i.d.\ pseudo-rewards $(Z_\ell)_{\ell = 1}^{2 a n} \sim \mathrm{Ber}(1 / 2)$. Note that both $X$ and $Y$ are random variables. Let $\bar{X} = \E{X}$ and $\bar{Y} = \E{Y}$. Our main theorem is stated and proved below.

\begin{theorem}
\label{thm:optimism upper bound} For any $a > 1$,
\begin{align*}
  & \E{1 / \condprob{X + Y \geq \bar{X} + \bar{Y}}{X}} \\
  & \quad \leq \frac{2 e^2 \sqrt{a}}{\sqrt{\pi}} \exp\left[\frac{8}{a - 1}\right]
  \left(1 + \sqrt{\frac{\pi a}{8 (a - 1)}}\right)\,.
\end{align*}
\end{theorem}
\begin{proof}
Let $W = \E{1 / \condprob{Y \geq \bar{X} - X + \bar{Y}}{X}}$. Note that $W$ can be rewritten as $W = \E{f(X)}$, where
\begin{align*}
  f(X)
  = \left[\smashoperator[r]{\sum_{y = \ceils{\bar{X} - X + a n}}^m}
  B(y; m, 1 / 2)\right]^{-1}
\end{align*}
and $m = 2 a n$. This follows from the definition of $Y$ and that $\bar{Y} = a n$.

Note that $f(X)$ decreases in $X$, as required by \cref{lem:history upper bound}, because the probability of observing at least $\ceils{\bar{X} - X + a n}$ ones increases with $X$ and $f(X)$ is its reciprocal. So we can apply \cref{lem:history upper bound} and get
\begin{align*}
  W
  \leq {} & \sum_{i = 0}^{i_0 - 1} \exp[-2 i^2]
  \left[\smashoperator[r]{\sum_{y = \ceils{a n + (i + 1) \sqrt{n}}}^m}
  B(y; m, 1 / 2)\right]^{-1} + {} \\
  & \exp[-2 i_0^2]
  \left[\smashoperator[r]{\sum_{y = \ceils{a n + \bar{X}}}^m}
  B(y; m, 1 / 2)\right]^{-1}\,,
\end{align*}
where $i_0$ is the smallest integer such that $(i_0 + 1) \sqrt{n} \geq \bar{X}$, as defined in \cref{lem:history upper bound}.

Now we bound the sums in the reciprocals from below using \cref{lem:perturbation lower bound}. For $\delta = (i + 1) \sqrt{n}$,
\begin{align*}
  \smashoperator[r]{\sum_{y = \ceils{a n + (i + 1) \sqrt{n}}}^m} B(y; m, 1 / 2)
  \geq \frac{\sqrt{\pi}}{e^2 \sqrt{a}}
  \exp\left[- \frac{2 (i + 2)^2}{a}\right]\,.
\end{align*}
For $\delta = \bar{X}$,
\begin{align*}
  \smashoperator[r]{\sum_{y = \ceils{a n + \bar{X}}}^m} B(y; m, 1 / 2)
  & \geq \frac{\sqrt{\pi}}{e^2 \sqrt{a}}
  \exp\left[- \frac{2 (\bar{X} + \sqrt{n})^2}{a n}\right] \\
  & \geq \frac{\sqrt{\pi}}{e^2 \sqrt{a}}
  \exp\left[- \frac{2 (i_0 + 2)^2}{a}\right]\,,
\end{align*}
where the last inequality is from the definition of $i_0$. Then we chain the above three inequalities and get
\begin{align*}
  W
  \leq \frac{e^2 \sqrt{a}}{\sqrt{\pi}}
  \sum_{i = 0}^{i_0} \exp\left[- \frac{2 a i^2 - 2 (i + 2)^2}{a}\right]\,.
\end{align*}
Now note that
\begin{align*}
  & 2 a i^2 - 2 (i + 2)^2 \\
  & \quad = 2 (a - 1) i^2 - 8 i - 8 \\
  & \quad = 2 (a - 1) \left(i^2 - \frac{4 i}{a - 1} +
  \frac{4}{(a - 1)^2} - \frac{4}{(a - 1)^2}\right) - 8 \\
  & \quad = 2 (a - 1) \left(i - \frac{2}{a - 1}\right)^2 - \frac{8 a}{a - 1}\,.
\end{align*}
It follows that
\begin{align*}
  W
  & \leq \frac{e^2 \sqrt{a}}{\sqrt{\pi}} \sum_{i = 0}^{i_0}
  \exp\left[- \frac{2 (a - 1)}{a} \left(i - \frac{2}{a - 1}\right)^2 +
  \frac{8}{a - 1}\right] \\
  & \leq \frac{2 e^2 \sqrt{a}}{\sqrt{\pi}} \exp\left[\frac{8}{a - 1}\right]
  \sum_{i = 0}^\infty \exp\left[- \frac{2 (a - 1)}{a} i^2\right] \\
  & \leq \frac{2 e^2 \sqrt{a}}{\sqrt{\pi}} \exp\left[\frac{8}{a - 1}\right] \!
  \left[1 + \!\!\! \smashoperator[r]{\int_{u = 0}^\infty}
  \exp\left[- \frac{2 (a - 1)}{a} u^2\right] \dif u\right] \\
  & = \frac{2 e^2 \sqrt{a}}{\sqrt{\pi}} \exp\left[\frac{8}{a - 1}\right]
  \left(1 + \sqrt{\frac{\pi a}{8 (a - 1)}}\right)\,.
\end{align*}
This concludes our proof.
\end{proof}

\begin{lemma}
\label{lem:history upper bound} Let $f(X)$ be a non-negative decreasing function of random variable $X$ in \cref{thm:optimism upper bound} and $i_0$ be the smallest integer such that $(i_0 + 1) \sqrt{n} \geq \bar{X}$. Then
\begin{align*}
  \E{f(X)}
  \leq {} & \sum_{i = 0}^{i_0 - 1} \exp[-2 i^2] f(\bar{X} - (i + 1) \sqrt{n}) + {} \\
  & \exp[-2 i_0^2] f(0)\,.
\end{align*}
\end{lemma}
\begin{proof}
Let
\begin{align*}
  \mathcal{P}_i
  =
  \begin{cases}
    (\max \set{\bar{X} - \sqrt{n}, \, 0}, \, n]\,, &
    i = 0\,; \\
    (\max \set{\bar{X} - (i + 1) \sqrt{n}, \, 0}, \, \bar{X} - i \sqrt{n}]\,, &
    i > 0\,;
  \end{cases}
\end{align*}
for $i \in [i_0] \cup \set{0}$. Then $\set{\mathcal{P}_i}_{i = 0}^{i_0}$ is a partition of $[0, n]$. Based on this observation,
\begin{align*}
  \E{f(X)}
  = {} & \sum_{i = 0}^{i_0} \E{\I{X \in \mathcal{P}_i} f(X)} \\
  \leq {} & \sum_{i = 0}^{i_0 - 1}
  f(\bar{X} - (i + 1) \sqrt{n}) \, \prob{X \in \mathcal{P}_i} + {} \\
  & f(0) \, \prob{X \in \mathcal{P}_{i_0}}\,,
\end{align*}
where the inequality holds because $f(x)$ is a decreasing function of $x$. Now fix $i > 0$. Then from the definition of $\mathcal{P}_i$ and Hoeffding's inequality, we have
\begin{align*}
  \prob{X \in \mathcal{P}_i}
  \leq \prob{X \leq \bar{X} - i \sqrt{n}}
  \leq \exp[-2 i^2]\,.
\end{align*}
Trivially, $\prob{X \in \mathcal{P}_0} \leq 1 = \exp[- 2 \cdot 0^2]$. Finally, we chain all inequalities and get our claim.
\end{proof}

\begin{lemma}
\label{lem:perturbation lower bound} Let $m = 2 a n$. Then for any $\delta \in [0, a n]$,
\begin{align*}
  \smashoperator[r]{\sum_{y = \ceils{a n + \delta}}^m} B(y; m, 1 / 2)
  \geq \frac{\sqrt{\pi}}{e^2 \sqrt{a}}
  \exp\left[- \frac{2 (\delta + \sqrt{n})^2}{a n}\right]\,.
\end{align*}
\end{lemma}
\begin{proof}
By Lemma 4 in Appendix of \citet{kveton19garbage},
\begin{align*}
  B(y; m, 1 / 2)
  \geq \frac{\sqrt{2 \pi}}{e^2} \sqrt{\frac{m}{y (m - y)}}
  \exp\left[- \frac{2 (y - a n)^2}{a n}\right]\,.
\end{align*}
Also note that
\begin{align*}
  \frac{y (m - y)}{m}
  \leq \frac{1}{m} \frac{m^2}{4}
  = \frac{a n}{2}
\end{align*}
for any $y \in [0, m]$. Now we combine the above two inequalities and get
\begin{align*}
  B(y; m, 1 / 2)
  \geq \frac{2 \sqrt{\pi}}{e^2 \sqrt{a n}}
  \exp\left[- \frac{2 (y - a n)^2}{a n}\right]\,.
\end{align*}
Finally, we note the following. First, the above lower bound decreases in $y$ for $y \geq an + \delta$, since $\delta \geq 0$. Second, by the pigeonhole principle, there are at least $\floors{\sqrt{n}}$ integers between $a n + \delta$ and $a n + \delta + \sqrt{n}$, starting with $\ceils{a n + \delta}$. This leads to the following lower bound
\begin{align*}
  \smashoperator[r]{\sum_{y = \ceils{a n + \delta}}^m} B(y; m, 1 / 2)
  & \geq \floors{\sqrt{n}} \frac{2 \sqrt{\pi}}{e^2 \sqrt{a n}}
  \exp\left[- \frac{2 (\delta + \sqrt{n})^2}{a n}\right] \\
  & \geq \frac{\sqrt{\pi}}{e^2 \sqrt{a}}
  \exp\left[- \frac{2 (\delta + \sqrt{n})^2}{a n}\right]\,.
\end{align*}
The last inequality is by $\floors{\sqrt{n}} / \sqrt{n} \geq 1 / 2$, which holds for $n \geq 1$. This concludes our proof.
\end{proof}

\bibliographystyle{named}
\bibliography{References}

\end{document}